\documentclass[letterpaper]{article}
\usepackage{uai2019}
\usepackage{times}
\usepackage[margin=1in]{geometry}
\usepackage{natbib}
\usepackage{hyperref}
\usepackage{enumitem}
\usepackage{multirow}
\usepackage{url}
\usepackage{graphicx}
\usepackage{subcaption}
\usepackage{xargs}
\usepackage{xcolor}
\usepackage[colorinlistoftodos,prependcaption,textsize=tiny]{todonotes}
\usepackage{amsmath,amssymb,amsfonts,amsthm}
\DeclareMathOperator*{\argmax}{arg\,max}

\newtheorem{claim}{Claim}[section]

\newtheorem{hyp}{Hypothesis}[section]

\theoremstyle{definition}
\newtheorem{definition}{Definition}[section]
\newtheorem{example}{Example}[section]

\newcommand{\R}[0]{\mathbb{R}}

\newcommand{\hpolicyspace}{\mathcal{H}}
\newcommand{\rpolicyspace}{\mathcal{R}}
\newcommand{\human}{\textbf{H}}
\newcommand{\ahuman}{\textbf{AH}}
\newcommand{\ah}{\ahuman}

\newcommand{\robot}{\textbf{R}}
\newcommand{\states}{\mathcal{S}}
\newcommand{\actions}{\mathcal{A}}
\newcommand{\hactspace}{\actions^\human}
\newcommand{\ractspace}{\actions^\robot}
\newcommand{\rspace}{\mathcal{R}}
\newcommand{\tr}{T}

\newcommand{\hpolicy}{H}
\newcommand{\rpolicy}{R}
\newcommand{\br}{\mathbf{BR}}
\newcommand{\ir}{\mathbf{IR}}
\newcommand{\val}{U}
\newcommand{\pr}{\mathbb{P}}
\newcommand{\dataset}{\mathcal{D}}
\newcommand{\lvar}{\theta}
\newcommand{\lvarspace}{\Theta}
\newcommand{\data}{x}
\newcommand{\dataspace}{\mathcal{X}}
\newcommand{\model}{m}
\newcommand{\pl}{\mathcal{L}_\dataspace}
\newcommand{\il}{\mathcal{L}_{\lvarspace}}

\title{Literal or Pedagogic Human? \\ Analyzing Human Model Misspecification in Objective Learning}
\author{Smitha Milli, Anca D. Dragan \\
University of California, Berkeley \\
\texttt{\{smilli,\,anca\}@berkeley.edu}}

\begin{document}

\maketitle

\begin{abstract}
   It is incredibly easy for a system designer to misspecify the objective for an autonomous system (``robot''), thus motivating the desire to have the robot \emph{learn} the objective from human behavior instead.  Recent work has suggested that people have an interest in the robot performing well, and will thus behave \emph{pedagogically}, choosing actions that are informative to the robot. In turn, robots benefit from interpreting the behavior by accounting for this pedagogy. In this work, we focus on misspecification: we argue that robots might not know whether people are being pedagogic or literal and that it is important to ask which assumption is \emph{safer} to make. We cast objective learning into the more general form of a common-payoff game between the robot and human, and prove that in any such game literal interpretation is \emph{more robust} to misspecification. Experiments with human data support our theoretical results and point to the sensitivity of the pedagogic assumption.
\end{abstract}

\section{INTRODUCTION}
\begin{figure*}[ht]
    \centering
    \includegraphics[width=\textwidth]{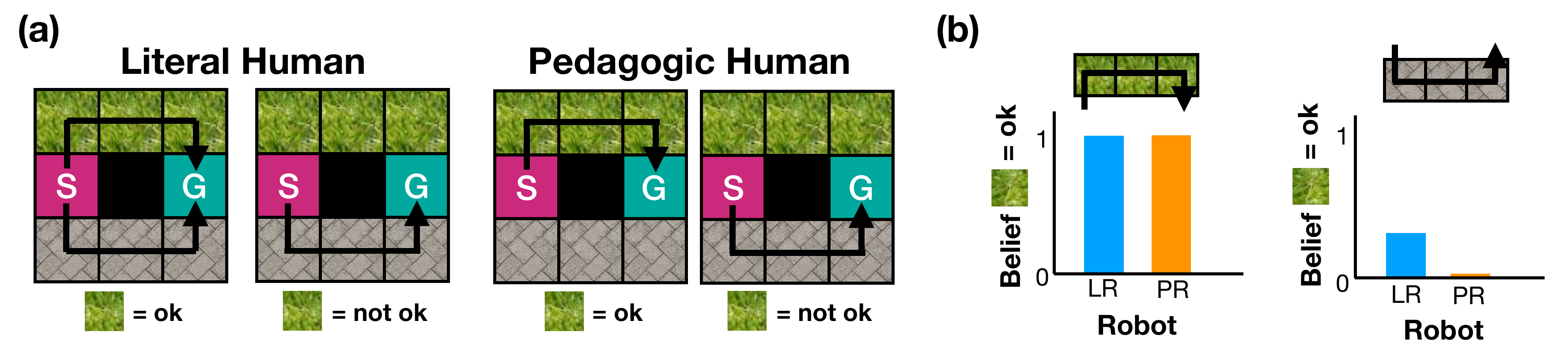}
    \caption{In this scenario, the robot infers whether it is ok to walk on the grass or not from a demonstration provided by the human (it already knows pavement is always ok). (a) The literal and pedagogic human's demonstrations. When grass is ok, the literal human is equally likely to walk on the grass or pavement. On the other hand, when grass is ok, the pedagogic human always walks on the grass in order to signal that grass is ok. (b) The beliefs of the literal robot \textbf{LR} and pedagogic robot \textbf{PR} after observing a demonstration. \textbf{LR} and \textbf{PR} assume human is literal and pedagogic, respectively. \textbf{LR} and \textbf{PR} differ in how strong their beliefs are after witnessing the human walk on pavement, leading to different problems when the human is misspecified. If the human is literal, but the robot is pedagogic, then it makes too strong of an inference. On the other hand, if the human is pedagogic, but the robot  is literal, then it makes too weak of an inference.}
    \label{fig:lit-vs-ped-h}
\end{figure*}
It is notoriously difficult for system designers to directly specify the correct objective for a system ~\citep{krakovna_2018, lehman2018surprising, clark_amodei_2017}. This difficulty has sparked a line of work that instead aims to infer the correct objective from other forms of human input, such as demonstrations ~\citep{ng2000algorithms,abbeel2004apprenticeship,ziebart2008maximum}, comparisons ~\citep{wirth2017survey, dorsa2017active,christiano2017deep}, or corrections ~\citep{bajcsy2017learning, jain2015learning}.

These methods operate under the assumption that human behavior is near-optimal with respect to the objective to be inferred. This assumption makes sense in many domains, like an autonomous car learning an objective function for driving through observing human drivers \citep{levine2012continuous}. However, recent work shows that this assumption may not hold in collaborative settings \citep{ho2016showing}, in which the human is \emph{aware} that the robot needs to learn. In such settings, the person might optimize for \emph{teaching} the robot about the objective, which is not the same as directly optimizing the objective itself \citep{Dragan:2013:TIC:3109708.3109711,hadfield2016cooperative,ho2016showing}. Figure \ref{fig:lit-vs-ped-h} shows an example of the difference between the two. When the human optimizes for the objective, she takes any optimal path to the goal. When the human optimizes for teaching the objective, she takes the path that best signals the objective.


We refer to these two types of human behavior as \emph{literal} and \emph{pedagogic}, and note that the robot can interpret human behavior using either model:
\begin{enumerate}[topsep=0pt,itemsep=1pt]
    \item The \textbf{literal human} directly optimizes for the objective.
    \item The \textbf{literal robot} infers the objective while assuming the human is literal.
    \item The \textbf{pedagogic human} optimizes for teaching the \emph{literal} robot the objective.
    \item The \textbf{pedagogic robot} (sometimes called ``pragmatic" \citep{fisac2018pragmatic}) infers the objective while assuming the human is \emph{pedagogic}.
\end{enumerate}
\vspace{-10pt}
  \begin{align*}
      \vdots
  \end{align*}
In general, this recursion could go on further, and in theory, it is always better for the human and robot to be at a deeper level of recursion. The potential for increased performance suggests that we should try to make our robots more pedagogic. Indeed, this is the direction suggested and pursued by \cite{fisac2018pragmatic,malik18,hadfield2016cooperative}


However, the increased performance is contingent on the human being pedagogic. In practice, it may be difficult to know whether the human is acting literally or pedagogically, and furthermore, different humans may behave differently. We argue that the robot should be robust to misspecification of the \emph{human}, and thus it is important to ask which assumption -- literal or pedagogic -- is \emph{safer} to make. 

Which is worse---interpreting a pedagogic human literally (literal robot + pedagogic human) or interpreting a literal human pedagogically (pedagogic robot + literal human)? In both cases, the model of the human is incorrect, and we might expect that which is better depends on the context and the task; surprisingly, we are able to prove that regardless of the task, the literal robot is more robust, suggesting that it may be safer to simply use a \emph{literal} robot. Our contributions are the following:


\begin{itemize}[topsep=0pt]
    \item  Section \ref{sec:theory}: \emph{Theoretical Analysis.} We cast objective learning into the more general form of a common-payoff game between the human and robot. We then prove that a pedagogic robot and a literal human always do worse than a literal robot and a pedagogic human, showing that misspecification is worse in one direction than the other. 
    \item  Section \ref{sec:experiments}: \emph{Empirical Analysis.} We test the effects of misspecification on data from human teaching. The data confirms that assuming pedagogic behavior when people are literal is worse than assuming literal behavior when people are pedagogic. Surprisingly, we find that the literal robot does better than the pedagogic robot \emph{even when people are trying to be pedagogic}, because of the discrepancy between the pedagogic \emph{model} of humans and real human behavior. 
    \item Section \ref{sec:theory-vs-emp}: \emph{Theoretical vs Empirical.} Our empirical results are surprising, in that, the pedagogic model is a state of the art cognitive science model that is fit to the human data and has relatively high predictive accuracy, yet the pedagogic robot does worse than the literal robot \emph{with humans who are trying to be pedagogic}. 
    We use our theory to derive a hypothesis for why this could be. The hypothesis is that in practice different humans vary in how literal or pedagogic they are, which our theory implies will degrade the performance of the pedagogic robot more than the literal robot. We find positive evidence for this hypothesis, indicating that robustness to a \emph{population} of humans is an important consideration in choosing a pedagogic versus literal robot.
    \item Section \ref{sec:fixing}: \emph{Can we ``fix'' the pedagogic robot?} An intuitive idea for improving the pedagogic robot is to give it a model of the pedagogic human that has higher predictive accuracy. For example, what if instead of assuming all people are pedagogic, the robot estimated how pedagogic each person is? Unfortunately, we find that this makes no difference. And in fact, we show that  better models can actually \emph{worsen} performance due to a subtle, yet remarkable fact: a human model with higher \emph{predictive accuracy} does not necessarily imply higher \emph{inferential accuracy} for the robot.
\end{itemize}

In conclusion, we found that not only are pedagogic robots less robust to the human's recursion level, they can also perform worse when people are actually being pedagogic -- even with a state of the art pedagogic model tuned to human data. This points to a surprising brittleness of the pedagogic assumption. Further, we point out that a more predictive human model will not necessarily solve the problem because improving the model's predictive accuracy does not necessarily lead to better robot inference.

The difficulty of \emph{objective specification} was an important motivation for pursuing objective learning in the first place. But in our pursuit of objective learning, we ought to be careful to not simply trade the problem of objective specification for the equally, if not more, difficult problem of \emph{human specification}. In practice, humans will deviate from our models of them, and thus, it is important to  understand which assumptions are robust and which are not.

Rather than trying to make the robot more pedagogic, which requires brittle assumptions on the human, an alternative direction may be to help the \emph{human} become more pedagogic. How can we make robots that are easier for humans to teach? Perhaps simpler robots are actually better, if it means that humans can more easily teach them.


\section{THEORETICAL ANALYSIS} \label{sec:theory}
In this section, we cast objective learning into the more general form of a common-payoff\footnote{A game in which all agents have the same payoff.} game between the robot and human. We then formalize the literal/pedagogic robot/human, and prove that in \emph{any} common-payoff game a literal robot and pedagogic human perform better than a pedagogic robot and literal human.

\subsection{GENERALIZING OBJECTIVE LEARNING TO CI(RL)}
Before proceeding to our proof, we give background on how common-payoff games generalize standard objective learning. In particular, objective learning can be modeled as a cooperative inverse reinforcement learning (CIRL) game \citep{hadfield2016cooperative}, a common-payoff game in which only the human knows an objective/reward function $r$. 

Formally, a CIRL game is defined as a tuple $\langle \states, \{\hactspace, \ractspace\}, \tr, \{\rspace, r\}, P_0, \gamma\}$ where $\states$ is the set of states, $\hactspace$ and $\ractspace$ are the set of actions available to the human and robot, $\tr(s' \mid s, a^\human, a^\robot)$ is the transition distribution specifying the probability of transitioning to a new state $s'$ given the previous state $s$ and the actions $a^\human$ and $a^\robot$ of both agents, $R$ is the space of reward functions, $r : \states \times \hactspace \times \ractspace \rightarrow \R$ is the shared reward function (known only to \human{}), $P_0$ is the initial distribution over states and reward functions, and $\gamma$ is the discount factor. 

The joint payoff in a CIRL game is traditionally \emph{value}, expected sum of rewards. Since only the human in CIRL knows the shared reward function, this indirectly incentivizes the human to act in ways that signal the reward function and incentivizes the robot to learn about the reward function from the human's behavior\footnote{In fact, in Demonstration-CIRL (Example \ref{ex:dcirl}), the best that the robot can do is infer a posterior distribution over rewards from the human's demonstrations, and then act optimally with respect to the posterior mean.}. However, we can also consider a version that directly incentivizes reward inference by making the payoff the accuracy of the robot's inference.

To illustrate how objective learning settings are special cases of CIRL games, we give an example for the learning from demonstrations setting. 

\begin{example}[Demonstration-CI(RL)]\label{ex:dcirl}
    Learning from demonstrations can be modeled as a CIRL game with two phases. In the first phase, the human provides demonstrations. In the second phase...
    \begin{enumerate}[label=(\alph*),topsep=0pt,noitemsep]
        \item (CIRL) the robot acts in the environment. The game's joint payoff for the robot and human is the value (expected sum of rewards) attained by the robot.
        \item (CI) the robot outputs an estimate of the reward function. The game's joint payoff for the robot and human is a measure of the accuracy of the robot's inference.
    \end{enumerate}
\end{example}

The second formulation (b) is a \emph{cooperative inference} (CI) problem \citep{yang2018optimal,wang2019generalizing}. In CI, there is a teacher (e.g human) who is teaching a learner (e.g. robot) a hypothesis $r$ (e.g the reward) via data $d$ (e.g. demonstrations). The human has a distribution over demonstrations $p^\human(d \mid r)$ and the robot has a distribution over rewards $p^\robot(r \mid d)$. Given a starting distribution of human demonstrations, $p^\human_0(d \mid r)$, the optimal solution for these two distributions can be found via fixed-point iteration of the following recursive equations\footnote{\cite{wang2019generalizing} show that fixed-point iteration converges for all discrete distributions. For simplicity, we have written equations (\ref{eq:r-ci}) and (\ref{eq:h-ci}) assuming that the human and robot's prior over rewards and demonstrations is uniform, but in general, the equations can incorporate any prior \citep{yang2018optimal}.}:
\begin{align}
    & p^{\robot}_{k}(r \mid d) \propto p^\human_{k}(d \mid r) \label{eq:r-ci} \,, \\
    & p^{\human}_{k+1}(d \mid r) \propto p^\robot_{k}(r \mid d) \label{eq:h-ci} \,.
\end{align}

\subsection{PROOF: LITERAL ROBOTS ARE MORE ROBUST}
We now proceed to formalize what we mean by literal and pedagogic and to prove that the literal robot is more robust to misspecification of whether the human is literal or pedagogic. Suppose that the human and robot are acting in a cooperative game. Let $\hpolicyspace$ and $\rpolicyspace$ be the space of policies for the human and robot, respectively. The joint payoff for the human and robot is denoted by $\val : \hpolicyspace \times \rpolicyspace \rightarrow \mathbb{R}$. The joint payoff function could be value, as assumed by CIRL, or accuracy of inference, as assumed by CI.

To define literal and pedagogic, we define a set of recursive policies for the human and robot. Let $\hpolicy_0$ be a starting human policy. For $k \geq 0$, define the recursive policies
\begin{align}
    \rpolicy_{k} & = \br(\hpolicy_k)\,, \label{eq:br-r} \\
    \hpolicy_{k+1} & = \ir(\rpolicy_{k})\,. \label{eq:ir-h}
\end{align} 
We assume that at each level of recursion the robot does a best response $\br : \hpolicyspace \rightarrow \rpolicyspace$. However, for the human, we only assume that at each step she improves over her previous policy. This allows for arbitrary irrationalities, so long as the next policy is at least as good as the previous one. We call this an ``improving'' response $\ir : \rpolicyspace \rightarrow \hpolicyspace$, and define both types of responses below.

\begin{definition}
A \emph{best response} for the robot is a function  $\br : \hpolicyspace \rightarrow \rpolicyspace$ such that for any human policy $\hpolicy \in \hpolicyspace$ and robot policy $\rpolicy \in \rpolicyspace$,
\begin{align*}
    \val(\hpolicy, \br(\hpolicy)) \geq \val(\hpolicy, \rpolicy) \,.
\end{align*}
\end{definition}

\begin{definition} \label{def:improving-response}
An \emph{improving response} for the human is a function  $\ir : \rpolicyspace \rightarrow \hpolicyspace$ such that $\forall \,k \geq 0$,
\begin{align*}
    \val(\ir(\rpolicy_k), \rpolicy_k) \geq \val(\hpolicy_k, \rpolicy_k) \,.
\end{align*}
\end{definition}
\begin{figure*}[t]
    \centering
    \includegraphics[width=\textwidth]{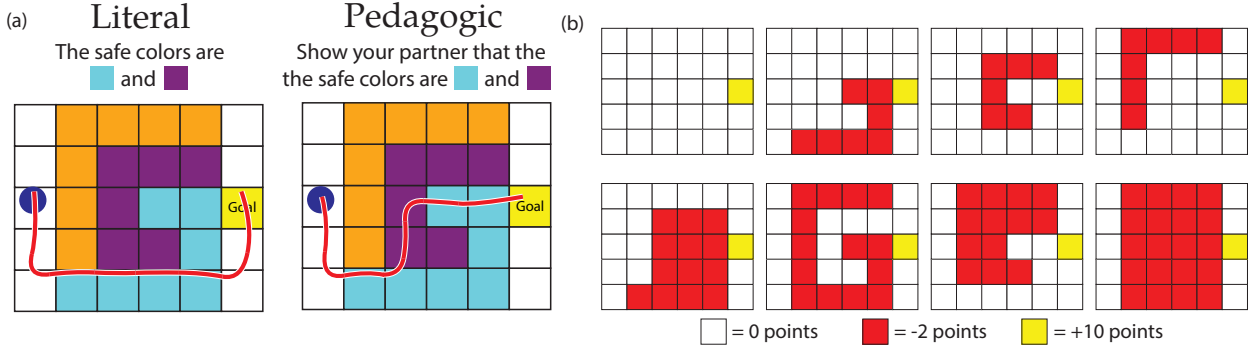}
    \caption{(a) The instructions given to participants in the literal and pedagogic condition, and a sample demonstration from both conditions. In the pedagogic case, participants were more likely to visit multiple safe colors, as well as to loop over safe tiles multiple times. (b) All possible reward functions. Each tile color can be either safe (0 points) or dangerous (-2 points). Figure modified from \cite{ho2018effectively}.}
    \label{fig:exp-details}
\end{figure*}

We give special emphasis to what we call the \emph{literal} human and robot, $\hpolicy_0$ and $\rpolicy_0$, and the \emph{pedagogic} human and robot, $\hpolicy_1$ and $\rpolicy_1$.\footnote{There is nothing that requires the literal level to be at recursion level $0$ and the pedagogic level to be at recursion level $1$. Our proof of Claim \ref{clm:thy-ranking} holds when the literal level is any $k \geq 0$ and the pedagogic level is $k+1$.} We now provide an example of the literal and pedagogic policies for the Demonstration-CI setting (Example \ref{ex:dcirl}b). In this case, the human and robot policies can be modeled by the recursive CI equations (\ref{eq:r-ci}) and (\ref{eq:h-ci}).
\begin{enumerate}[topsep=0pt]
    \item The \textbf{literal human} $H_0$ is noisily-optimal with respect to the reward function $r$. The probability $p^{\human}_0(d \mid r)$ that she gives a demonstration $d$ is exponentially proportional to the reward of the demonstration, denoted by $r(d)$:
    \begin{align*}
        p^{\human}_0(d \mid r) \propto \exp (r(d)) \,.
    \end{align*}
    \item The \textbf{literal robot} $R_0$ does a $\br$ to the literal human, i.e. does Bayesian inference assuming the human is literal,
    \begin{align*}
        p^{\robot}_0(r \mid d) \propto p^{\human}_0(d \mid r) \,,
    \end{align*}
    and then uses the posterior mode as its estimate for the reward $r$.
    \item The \textbf{pedagogic human} $H_1$ picks demonstrations that are informative to the literal robot\footnote{Typically, the human is modeled as being exponentially more informative: $p^{\human}_1(d \mid r) \propto \exp(p^{\robot}_0(r \mid d))$. This is the form we will use in our experiments.}:
    \begin{align*}
        p^{\human}_1(d \mid r) \propto p^{\robot}_0(r \mid d) \,.
    \end{align*}
    Note this is not a best response, which would unrealistically require the human to choose $\argmax_d p^{\robot}_0(r \mid d)$.
    \item The \textbf{pedagogic robot} $R_1$ does a $\br$ to the pedagogic human, i.e. does Bayesian inference assuming the human is literal,
    \begin{align*}
        p^{\robot}_1(r \mid d) \propto p^{\human}_1(d \mid r) \,,
    \end{align*}
    and then uses the posterior mode as its estimate for the reward $r$.
\end{enumerate}

We show that for \emph{any} common-payoff game the payoffs for the literal/pedagogic human/robot pairs have the following ranking:
    \begin{align}
        & \text{\small{Pedagogic $R_1$, $H_1$}} \geq \text{\small{Literal $R_0$, Pedagogic $H_1$}} \nonumber \\
        & \geq \text{\small{Literal $R_0$, $H_0$}} 
        \geq  \text{\small{Pedagogic $R_1$, Literal $H_0$}} \,. \label{eq:theory-ranking}
    \end{align}
In particular, a pedagogic robot $\rpolicy_1$ and a literal human $\hpolicy_0$ always do worse than a literal robot $\rpolicy_0$ and a pedagogic human $\hpolicy_1$, showing misspecification is worse one way than the other. The ranking has the following straight-forward proof.
\begin{claim} \label{clm:thy-ranking}
In any common-payoff game, the ranking of payoffs between a literal/pedagogic human/robot is
\begin{align*}
    \val(\hpolicy_1, \rpolicy_1) \geq \val(\hpolicy_1, \rpolicy_0) \geq \val(\hpolicy_0, \rpolicy_0) \geq \val(\hpolicy_0, \rpolicy_1) \,.
\end{align*}
\begin{proof}
Since $\rpolicy_1 = \br(\hpolicy_1)$, we have $\val(\hpolicy_1, \rpolicy_1) \geq \val(\hpolicy_1, \rpolicy_0)$. Since $\hpolicy_1 = \ir(\rpolicy_0)$, we have $\val(\hpolicy_1, \rpolicy_0) \geq \val(\hpolicy_0, \rpolicy_0)$. Since $\rpolicy_0 = \br(\hpolicy_0)$, we have $\val(\hpolicy_0, \rpolicy_0) \geq \val(\hpolicy_0, \rpolicy_1)$.
\end{proof}
\end{claim}

\section{EMPIRICAL ANALYSIS} \label{sec:experiments}


In this section, we test our results in practice using experimental data from actual humans, who provide demonstrations that the robot uses to infer the objective from. These experiments are a way of stress-testing our theoretical results, which assumed that at each level of recursion, the robot computes a best response and that the human never does worse than her previous level. In practice, there will be a difference between what humans actually do, and the model of the human, which could cause both the human and robot to break our theoretical assumptions.



\subsection{EXPERIMENTAL DESIGN}
We test the performance of literal/pedagogic robot/human pairs on data from Experiment 2 in \citep{ho2016showing}, which was subsequently followed up on in \citep{ho2018effectively}. In the experiment, humans are asked to act in different types of gridworlds. The gridworld has one goal state that is worth 10 points and three other types of tiles (orange, purple, cyan). Each type of tile can each be either ``safe" (0 points) or ``dangerous" (-2 points). Thus, there are $2^3 = 8$ possible reward functions, which are depicted in Figure \ref{fig:exp-details}b.

Sixty particpants were recruited from Mechanical Turk. The participants  are told that they will get two cents of bonus for each point they get. Participants were split into two conditions, a literal and pedagogic condition, depicted in Figure \ref{fig:exp-details}a. In the literal condition, the participant only gets points for their own actions in the gridworld. In the pedagogic condition, the participant is told that their demonstration will be shown to another person, a learner, who will then apply what they learn from the demonstration to act in a separate gridworld. The participant still gets points based on their own actions, but is also told that the number of points the learner receives will be added as a bonus.

\subsection{HUMAN AND ROBOT MODELS} \label{sec:hr-models}
We model the robot and human following \cite{ho2018effectively}. We use the same model parameters that \cite{ho2018effectively} found to be the best qualitative match to the human demonstrations.

\emph{Notation.} Let $\states$ be the set of states and $\actions$ be the set of actions. The gridworld has a reward function $r : \states \times \actions \times \states \rightarrow \R$. The optimal $Q$-value function for a reward function $r$ is denoted by $Q^*_r : \states \times \actions \rightarrow \R$. The robot has a uniform prior belief over the reward function, i.e, it puts uniform probability on all $2^3 = 8$ reward functions depicted in Figure \ref{fig:exp-details}b. The human and robot models are as follows.

\begin{enumerate}[topsep=0pt]
    \item \textbf{Literal \human{}.}  At each time step $t$, the probability $H_L(a_t|s_t, r)$ that the literal human takes action $a_t$ given state $s_t$ and the reward $r$ is exponentially proportional to the optimal $Q$-value,
        \begin{align} \label{eq:lit-human}
            H_L(a_t \mid s_t, r) \propto \exp(Q^*_r(s_t, a_t)/\tau_L) \,.
        \end{align}
    The temperature parameter $\tau_L$ controls how noisy the human is.
    \item \textbf{Literal \robot{}.} The robot does Bayesian inference, while assuming that the human is literal. The robot's posterior belief at time $t+1$ is
        \begin{align} 
            R_L^{t+1}(r) \propto & ~H_L(a_t \mid s_t, r) \nonumber \\
            & \cdot T(s_{t+1} \mid s_t, a_t) \cdot R_L^t(r) \,. \label{eq:lit-robot}
        \end{align}
    \item \textbf{Pedagogic \human{}.} The pedagogic human optimizes a reward function $r'$ that trades-off between optimizing the reward in the gridworld and teaching the literal robot the reward. At time step $t$, the reward is $r'(s_t, a_t, s_{t+1}) = r(s_t, a_t, s_{t+1}) + \kappa (R_L^{t+1}(r) - R_L^{t}(r))$. The parameter $\kappa \geq 0$ controls how ``pedagogic'' the human is. The probability $H_P(a_t \mid s_t, r)$ that the pedagogic human takes action $a_t$ given state $s_t$ and the reward $r$ is exponentially proportional to the optimal $Q$-value associated with the modified reward $r'$,
        \begin{align} \label{eq:ped-human}
            H_P(a_t \mid s_t, r) \propto \exp(Q^*_{r'}(s_t, a_t)/\tau_P) \,.
        \end{align}
    The temperature parameter $\tau_P$ controls how noisy the human is.
    \item \textbf{Pedagogic \robot{}.} The robot does Bayesian inference, while assuming that the human is literal. The robot's posterior belief at time $t+1$ is
        \begin{align} 
            R_P^{t+1}(r) \propto & ~H_P(a_t \mid s_t, r) \nonumber \\
            & \cdot T(s_{t+1} \mid s_t, a_t)\cdot R_P^t(r) \,. \label{eq:ped-robot}
        \end{align}
\end{enumerate}

\subsection{RESULTS}
We evaluate each literal/pedagogic robot/human pair on the accuracy of the robot's inference, i.e, $\pr(\hat{r} = r)$, where $r$ is the true reward and $\hat{r}$ is the robot's guess. We take the robot's guess $\hat{r}$ to be the mode of its belief, as given by the robot models (\ref{eq:lit-robot}) and (\ref{eq:ped-robot}). We test each pair with both the demonstrations generated by actual humans and demonstrations generated by simulating humans according to the human models (\ref{eq:lit-human}) and (\ref{eq:lit-robot}).


\begin{figure}[t]
    \centering
    \includegraphics[width=.5\textwidth]{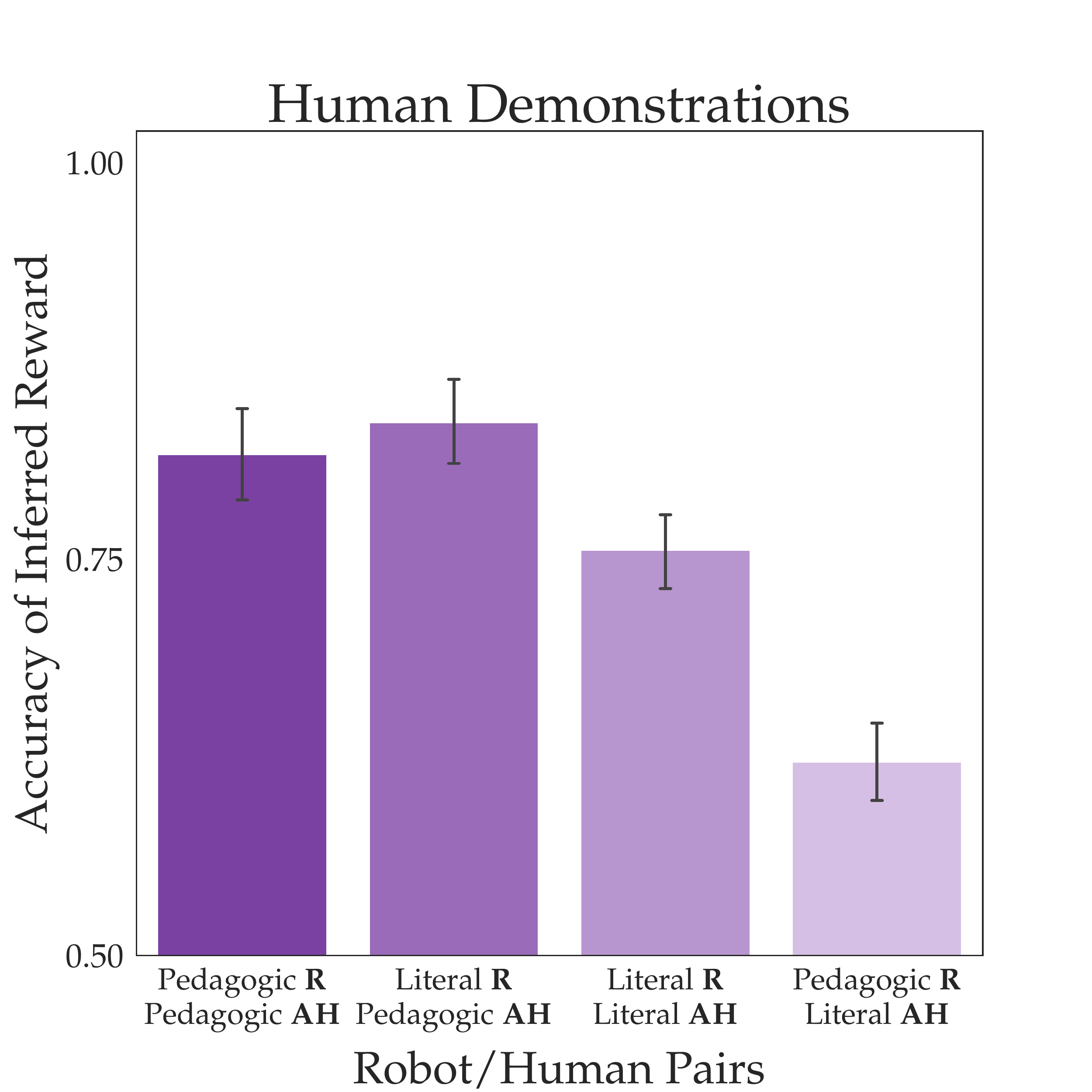}
    \caption{The accuracy of the robot's inferred reward for different pairs of human/robot pairs under the demonstrations provided by actual humans.}
    \label{fig:human-results}
\end{figure}

Figure \ref{fig:human-results} depicts our experimental results\footnote{All error bars and confidence bands in the paper depict bootstrapped 95\% confidence intervals.}. We refer to the actual human as \ahuman{}, the human model as \human{}, and the robot as \robot{}. Consistent with the theory, the performance of pedagogic \robot{} and literal \ahuman{} is (significantly) worse than that of literal \robot{} and pedagogic \ahuman{}, validating that misspecification is worse one way than the other. However, the overall ranking of robot/human pairs does not match the theoretical ranking (Equation \ref{eq:theory-ranking}) we expected. Surprisingly, even when \ah{} is pedagogic, pedagogic \robot{} performs (insignificantly) worse than literal \robot{}. The empirical ranking is
    \begin{align*}
        & \text{\small{Literal \robot, Pedagogic \ahuman}} \geq \text{\small{Pedagogic \robot, \ahuman}} \\
        & \geq \text{\small{Literal \robot, \ahuman}} \geq  \text{\small{Pedagogic \robot, Literal \ahuman}} \,.
    \end{align*}
The empirical ranking carries a stronger implication than our theoretical ranking---it implies that regardless of whether the human is literal or pedagogic, the robot should be literal. 

\section{THEORETICAL VS EMPIRICAL}
\begin{figure}[t]
    \centering
    \includegraphics[width=.5\textwidth]{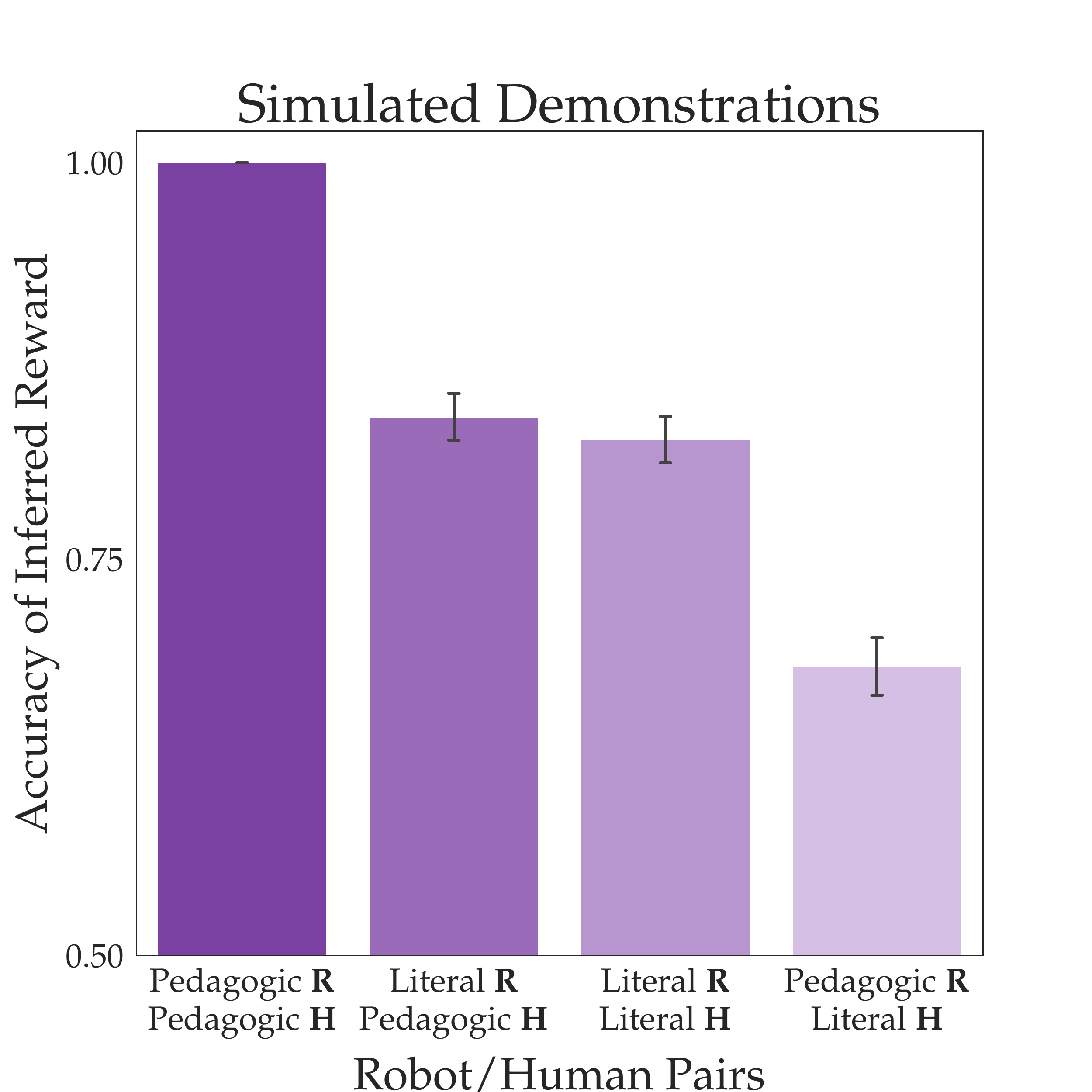}
    \caption{The accuracy of the robot's inferred reward for different pairs of human/robot pairs under the demonstrations provided by humans simulated according to the human models in Section \ref{sec:hr-models}.}
    \label{fig:sim-results}
\end{figure}

\begin{figure*}[t]
\centering
\begin{subfigure}[t]{.5\textwidth}
  \centering
  \includegraphics[width=\textwidth]{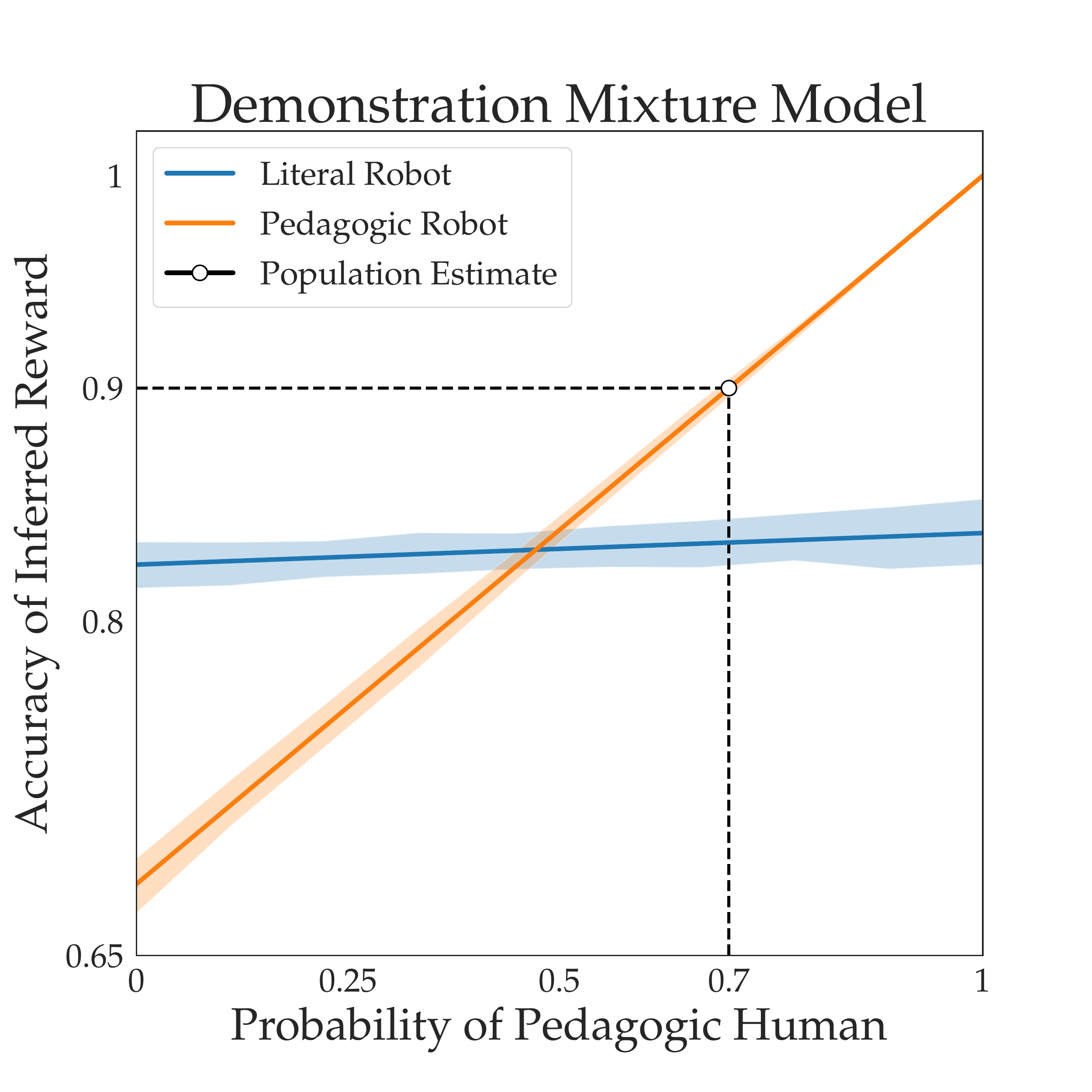}
\end{subfigure}%
\begin{subfigure}[t]{.5\textwidth}
  \centering
  \includegraphics[width=\textwidth]{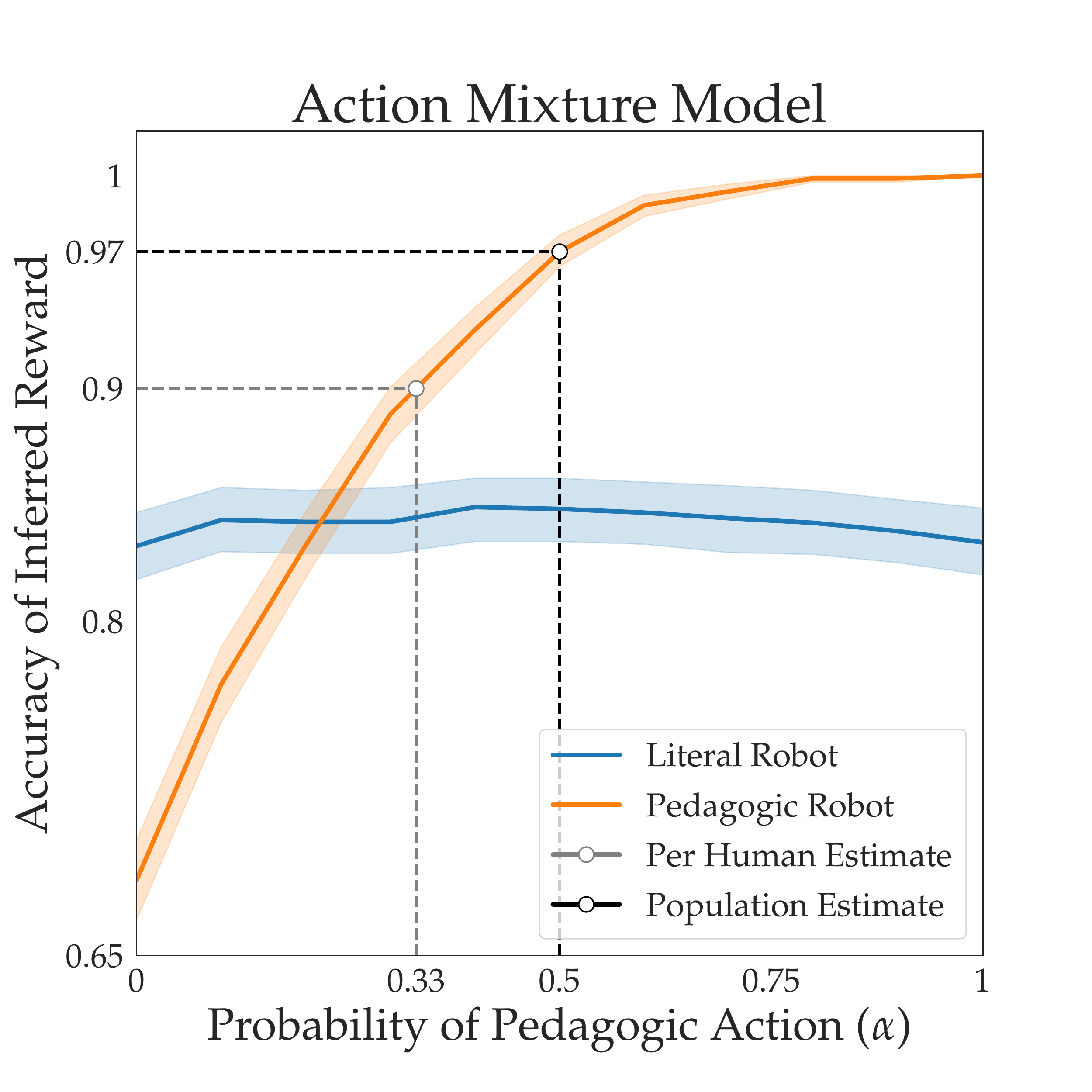}
  \label{fig:sub1}
\end{subfigure}
\caption{The performance of the literal and pedagogic robot when demonstrations are simulated according to the demonstration mixture model (left/a) or the action mixture model (right/b).}
\label{fig:mixture-results}
\end{figure*}

\label{sec:theory-vs-emp}
In our empirical analysis, we found a perplexing result: the literal robot does better than the pedagogic robot \emph{even when people are trying to be pedagogic}. How could this be? Figure \ref{fig:sim-results} shows the accuracy of robot/human pairs when the human demonstrations are simulated from the literal \human{} and pedagogic \human{} models described in equations (\ref{eq:lit-human}) and (\ref{eq:ped-human}). In the simulations, we find, as expected, that pedagogic \robot{} and pedagogic \human{} do better than literal \robot{} and pedagogic \human{}. So clearly, the reason pedagogic \robot{} does worse in the human experiments is that pedagogic \ah{} (actual humans) is not the same as the pedagogic \human{} (the human model). 


But, in \emph{what way} does pedagogic \ah{} deviate from pedagogic \human{}? Why is literal \robot{} more robust to the deviation than pedagogic \robot{}? We derive a hypothesis from our theory. In the theoretical ranking (Equation \ref{eq:theory-ranking}), the performance of literal \robot{} and literal/pedagogic \human{} is sandwiched between the performance of pedagogic \robot{} + pedagogic \human{} and pedagogic \robot{} + literal \human{}. Thus, literal \robot{} is more robust to whether \human{} is literal or pedagogic than pedagogic \robot{}. This suggests an explanation for why literal \robot{} does better, namely, that pedagogic \ah{} is actually sometimes literal! 

\begin{hyp} \label{exp}
Pedagogic \ah{} is actually ``in between" literal \human{} and pedagogic \human{}. If this is true, then our theory implies that literal \robot{} will be affected less than pedagogic \robot{}, potentially explaining why pedagogic \robot{} does worse with pedagogic \ah{} than literal \robot{}.
\end{hyp}

To test Hypothesis \ref{exp}, we create two models that are mixtures of the literal and pedagogic model. We then measure how much better the mixture models are at predicting pedagogic \ah{}, and how much the literal and pedagogic \robot{} are affected by demonstrations generated from the mixture models.


\subsection{DEMONSTRATION MIXTURE MODEL}
First, we test what we call the ``demonstration mixture model''. It is possible that the humans in the pedagogic condition from \cite{ho2016showing} followed different strategies; some may have attempted to be pedagogic, but others may have simply been literal. If we look at which model, literal or pedagogic, is a better fit on a per-individual basis, we find that 89.7\% of literal \ah{} are better described by literal \human{}, but in comparison, only 70.0\% of pedagogic \ah{} are better described by pedagogic \human{}. If we simulate pedagogic humans as only following the model 70.0\% of the time, then the accuracy of the pedagogic robot drops to 90\% (Figure \ref{fig:mixture-results}a). The literal robot is hardly impacted.

\subsection{ACTION MIXTURE MODEL} \label{sec:action-mixture}
\begin{figure*}[t]
\centering
\begin{subfigure}[t]{.5\textwidth}
  \centering
  \includegraphics[width=\textwidth]{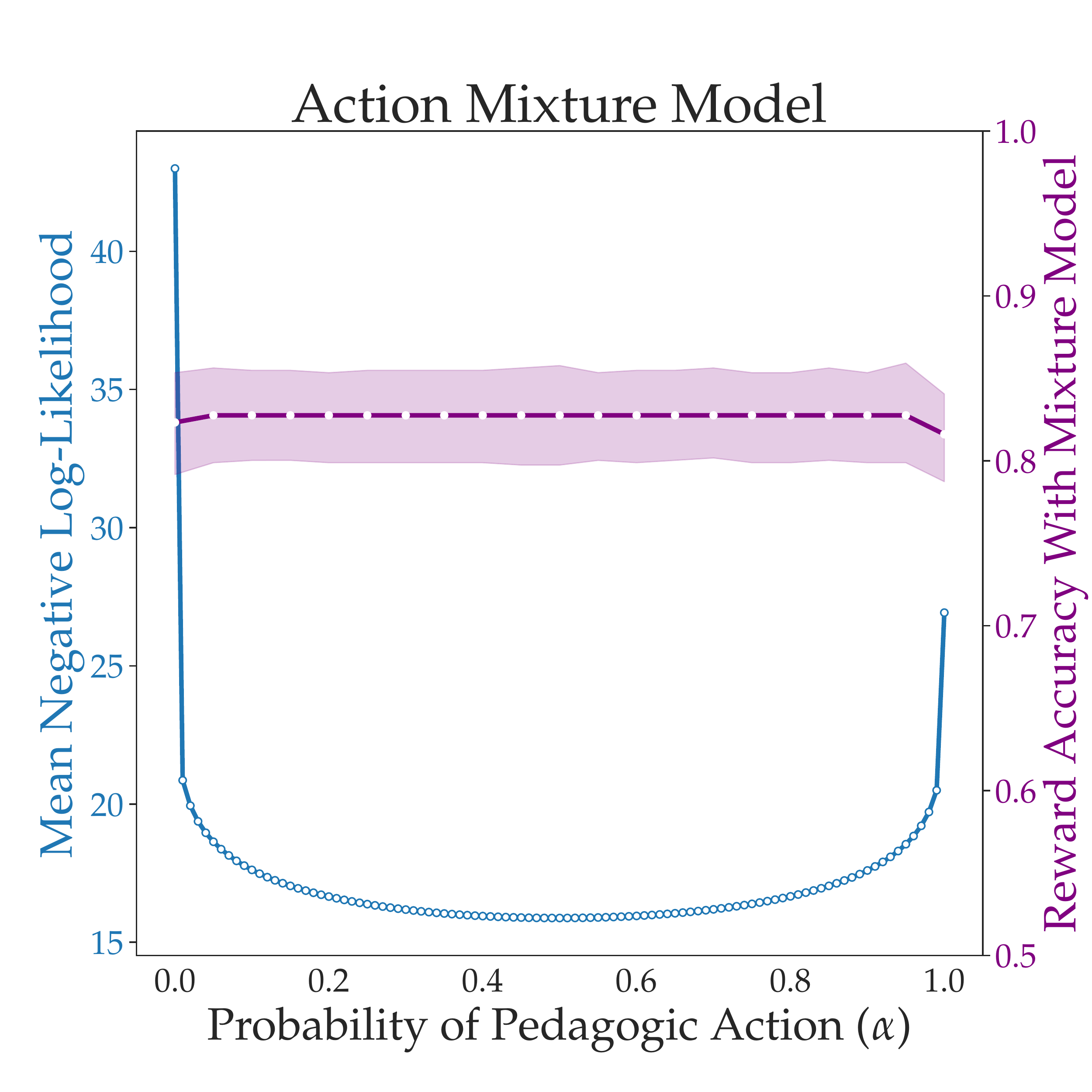}
\end{subfigure}%
\begin{subfigure}[t]{.5\textwidth}
  \centering
  \includegraphics[width=\textwidth]{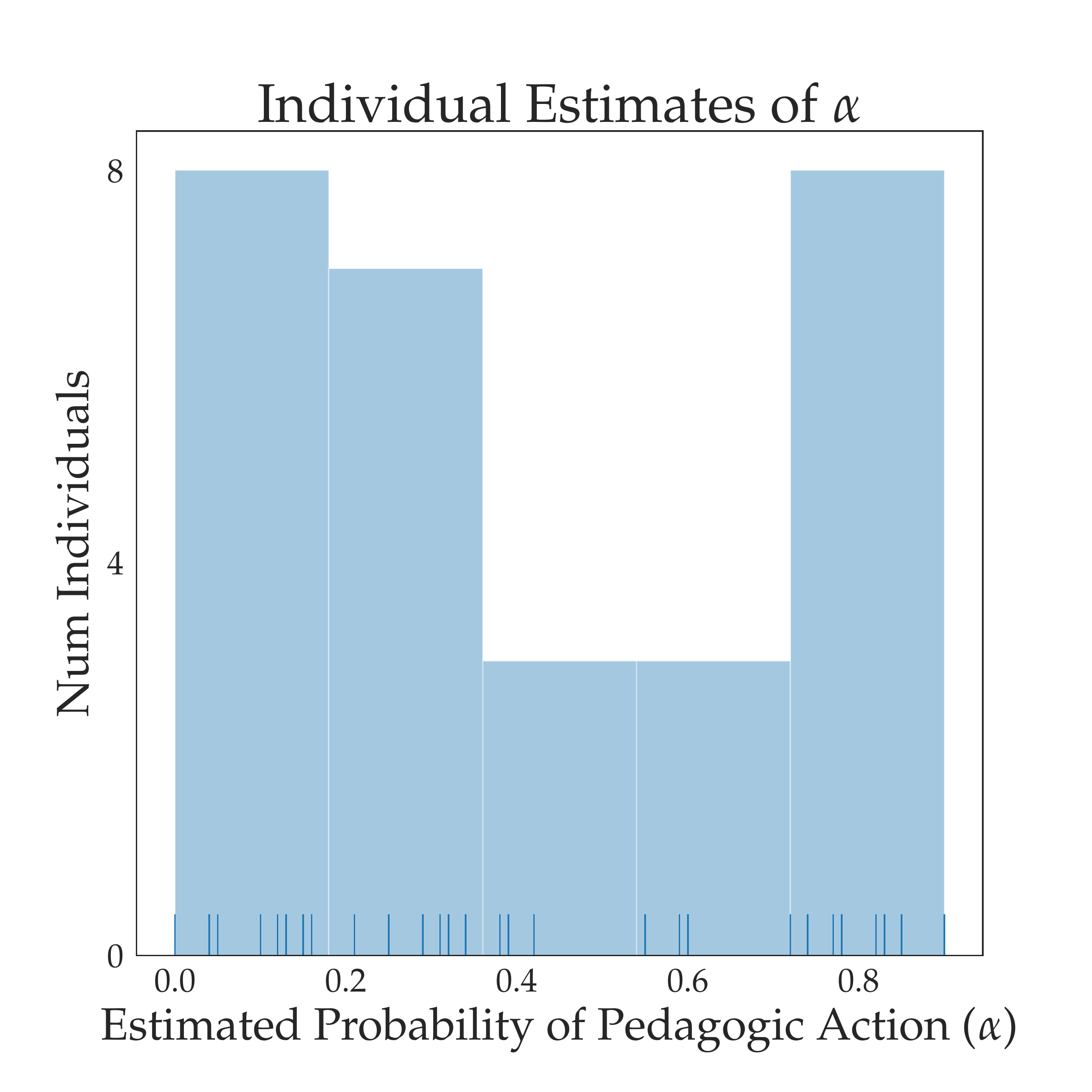}
\end{subfigure}
\caption{(left/a) In blue is the mean negative log-likelihood of the pedagogic human demonstrations under the action mixture model with probability $\alpha$. All individuals are assumed to have the same value of $\alpha$. Note that a mixture between the literal and pedagogic model is far better than either the literal model ($\alpha = 0$) or the pedagogic model ($\alpha = 1$). In purple is the accuracy of the robot's inferred reward, when given demonstrations from pedagogic \ahuman{}, if it assumes the human acts according to the action mixture model. (right/b) The mixture probability $\alpha$ that maximized log-likelihood for each individual pedagogic \ahuman{}.}
\label{fig:alpha-fits}
\end{figure*}
We now test a more continuous version of the previous setting. We now model humans as acting according to a mixture policy $H_m$ between the literal $H_L$ and pedagogic $H_P$ policies. In particular, at each time $t$ the probability the human picks action $a_t$ from state $s$ in an environment with reward $r$ is
\begin{align}
    & H_{M}(a_t \mid s_t, r) = \nonumber \\ 
    & \alpha H_{P}(a_t \mid  s_t, r) + (1-\alpha) H_{L}(a_t \mid s_t, r) \,. \label{eq:action-mixture}
\end{align}
The parameter $\alpha$ is the probability of picking an action according to the pedagogic model. We plot the likelihood of the actual pedagogic human demonstrations as a function of $\alpha$, $(\ref{eq:action-mixture})$. The best value of $\alpha$ over the whole population was $\alpha_P = 0.5$ (Figure \ref{fig:alpha-fits}a). But surprisingly, even a mixture with $\alpha = 0.01$ or $\alpha = 0.99$ is far better than either the literal ($\alpha = 0$) or pedagogic ($\alpha = 1$) model. In addition, the best estimates for $\alpha$ on an individual basis are somewhat bimodal (Figure \ref{fig:alpha-fits}b), indicating that there is high individual variation.

Figure \ref{fig:mixture-results}b shows the performamce of the literal and pedagogic robot as $\alpha$ varies. At the best population level $\alpha_P = 0.5$, we find that the pedagogic robot gets 97\% accuracy. However, if we simulate the pedagogic humans using the values of $\alpha$ estimated at an individual level, then the pedagogic robot's accuracy drops to 90\%, highlighting the importance of individual variation. The literal robot again remains unaffected.

\subsection{DISCUSSION}
In both the demonstration and action mixture models, we found that a mixture between pedagogic \human{} and literal \human{} was a much better fit to \ahuman{}. In both cases, when pedagogic \human{} is simulated according to the more accurate mixture model pedagogic \robot{}'s accuracy drops ten percentage points, but literal \robot{}'s accuracy hardly changes. Thus our results provide positive evidence for Hypothesis \ref{exp}. 

However, Hypothesis \ref{exp} does not explain the full story. The hypothesis can only account for a ten percentage drop in accuracy, but even with this drop, pedagogic \robot{} would still be better than literal \robot{}. Furthermore, with a cursory glance at Figure \ref{fig:mixture-results}, one might be tempted to consider pedagogic \robot{} quite robust, as it remains high-performing for large ranges of $\alpha$. However, as usual, the real problem is the \emph{unknown unknowns}. Our empirical results imply that there are other ways that humans deviate from the model and that literal \robot{} is more robust than pedagogic \robot{} to these unknown deviations. Rather than robustness to $\alpha$, the more compelling reason for choosing to use literal \robot{} is robustness to these unknown deviations.

\section{CAN WE ``FIX'' THE PEDAGOGIC ROBOT?} \label{sec:fixing}
An intuitive idea for improving the performance of the pedagogic robot \robot{} is to give it a model of the pedagogic human \human{} that is more predictive. Unfortunately, it is not so simple. For example, in Section \ref{sec:action-mixture}, we showed that a mixture between the literal and pedagogic human model was a much better fit to actual pedagogic humans than either the literal or pedagogic human model. What if we had pedagogic \robot{} use the more accurate mixture model as its model of the pedagogic human? Unfortunately, as shown in Figure \ref{fig:mixture-results}a (purple line), even if pedagogic \robot{} uses a better predictive model, i.e. the action mixture model, it does not perform any better.

In fact,  it is possible for pedagogic \robot{} to do \emph{worse} when given a \emph{more} predictive model of the human. The reason is that a model that is better for \emph{predicting} behavior (e.g. human demonstrations) is not necessarily better for \emph{inferring} underlying, latent variables (e.g. the reward function), which is what is important for the robot. This may help explain why pedagogic \robot{} does worse than literal \robot{} with pedagogic \ah{}, even though the pedagogic \human{} model assumed by pedagogic \robot{} is more predictive of pedagogic \ah{} than the literal \human{} model assumed by literal \robot{}.

To illustrate, suppose there is a latent variable $\lvar \in \lvarspace$ with prior distribution $p(\lvar)$ and observed data $\data \in \dataspace$ generated by some distribution $p(\data \mid \lvar)$. In our setting, $\lvar$ corresponds to the objective and $\data$ corresponds to the human input. For simplicity, we assume $\lvarspace$ and $\dataspace$ are finite. We have access to a training dataset $\dataset = \{(\lvar_i, \data_i)\}_{i=1}^n$ of size $n$. A \emph{predictive model} $m(\data \mid \theta)$ models the conditional probability of the data $\data$ given latent variable $\theta$ for all $\data \in \dataspace, \lvar \in \lvarspace$. In our case, the predictive model is the model of the human. The \emph{predictive likelihood} $\pl$ of a predictive model $m$ is simply the likelihood of the data under the model: 
\begin{align} \label{eq:pred-likelihood}
    \pl(m) = \prod_{i=1}^n m(\data_i \mid \lvar_i)\,.
\end{align}
The \emph{inferential likelihood} is the likelihood of the latent variables after applying Bayes' rule: 
\begin{align} \label{eq:inf-likelihood}
    \il(m) = \prod_{i=1}^{n} \frac{m(\data_i \mid \theta_i)p(\theta_i)}{\sum_{\lvar}m(\data_i \mid \lvar)p(\lvar)} \,.
\end{align}
Next, we show higher predictive likelihood does not necessarily imply higher inferential likelihood.
\begin{claim}[Predictive vs inferential likelihood] \label{clm:pred-vs-inf}
There exist settings in which there are two predictive models $m_1, m_2$ such that $\pl(m_1) > \pl(m_2)$, but $\il(m_1) < \il(m_2)$.
\end{claim}
\begin{proof}

Suppose that $\lvarspace = \{\lvar_1, \lvar_2\}$ and $\mathcal{X} = \{x_1, x_2, x_3\}$, the prior $p(\lvar)$ is uniform over $\lvarspace$, and the dataset $\dataset$ contains the following $n=9$ items:
\begin{align*}
\mathcal{D} =
\{& (\theta_1, x_1), (\theta_1, x_1), (\theta_1, x_2),
(\theta_2, x_2), (\theta_2, x_2),\\
& (\theta_2, x_3), (\theta_2, x_3), (\theta_2, x_3), (\theta_2, x_3)\} \,.
\end{align*}
    Define the models $m_1(\data \mid \lvar)$ and $m_2(\data \mid \lvar)$ by the following conditional probabilities tables.
    {
    
    \centering
    \begin{tabular}{|c|c|c|c|}
    \multicolumn{4}{c}{$\model_1(\data \mid \lvar)$} \\
        \hline  & $x_1$ & $x_2$ & $x_3$ \\
        \hline $\lvar_1$ & 2/3 & 1/3 & 0 \\
        \hline $\lvar_2$ & 0 & 1/3 & 2/3 \\ \hline
    \end{tabular}
    \quad 
    \begin{tabular}{|c|c|c|c|}
    \multicolumn{4}{c}{$\model_2(\data \mid \lvar)$} \\
        \hline  & $x_1$ & $x_2$ & $x_3$ \\
        \hline $\lvar_1$ & 2/3 & 1/3 & 0 \\
        \hline $\lvar_2$ & 0 & 2/3 & 1/3 \\ \hline
    \end{tabular}
    
    }
    The model $m_1$ has predictive likelihood $\pl(\model_1) = (2/3)^6(1/3)^3$ and inferential likelihood $\il(\model_1) = (1/2)^3$. The model $\model_2$ has predictive likelihood $\pl(\model_2) = (2/3)^4(1/3)^5$ and inferential likelihood $\il(\model_2) = (1/3)(2/3)^3$. Thus, $\pl(\model_1) > \pl(\model_2)$, but $\il(\model_1) < \il(\model_2)$.
\end{proof}
In our context, Claim \ref{clm:pred-vs-inf} means that a human model that is better in terms of \emph{prediction} may actually be worse for the robot to use for \emph{inference}. An alternative approach could be to directly fit models that optimize for inferential likelihood. Unfortunately, this optimization becomes much trickier because of the normalization over the latent variable space $\lvarspace$ in the denominator of (\ref{eq:inf-likelihood}), see e.g. \citep{dragan2013legibility}.

\section*{Acknowledgements}
We thank John Miller, Rohin Shah, and Mark Ho for providing feedback on a draft of the paper. In addition, we thank Mark Ho for his helpfulness in answering any questions about his work, for sharing his code, and for points of discussion along the way.

The work in this paper was supported by the National Science Foundation National Robotics Initiative and the National Science
Foundation Graduate Research Fellowship Program under Grant
No. DGE 1752814.
\bibliography{main.bib}
\bibliographystyle{iclr2019_conference}

\end{document}